%% file: main.tex
\documentclass[runningheads, envcountsame]{llncs}
\usepackage[T1]{fontenc}
\usepackage{graphicx}
\input{preamble/packages}
\input{preamble/macros}

\begin{document}
\title{Improved Runtime Bound for the \mpoEA on \binval}
%
%
\author{Joris Belder\orcidID{0009-0007-5525-3493} \and\\ Johannes Lengler\orcidID{0000-0003-0004-7629} \and\\
Raghu Raman Ravi\orcidID{0000-0003-3641-1824}}
\authorrunning{J. Belder, J. Lengler, R. Ravi}
%
\institute{ETH Z\"{u}rich, Switzerland \\
\email{jbelder@student.ethz.ch, \{johannes.lengler,raghu.ravi\}@inf.ethz.ch}.}
\maketitle              
\begin{abstract}
We study the $(\mu+1)$~EA on the Binary Value function \binval. We show that it needs at most $O(\mu \log \mu \cdot n \log n)$ function evaluations to find the optimum when $\mu = o(n/\log n)$. This substantially improves upon the recent upper bound of $O(\mu^5 n \log(n/\mu^4))$ by Krejca, Neumann and Witt. Our results hold for several mutation operators including standard bit mutation. In particular, our bound implies that the $(\mu+1)$~EA is at most a factor $O(\log \mu \cdot \log n)$ slower on \binval than on \textsc{OneMax}. 
\end{abstract}

\input{chapters/00-introduction}
\input{chapters/01-preliminaries}
\input{chapters/02-generalisation}
\input{chapters/03-zeroonemutation}
\input{chapters/04-standardmutation}

\input{chapters/05-experiments}
\input{chapters/06-conclusion}

\begin{credits}
\subsubsection{\ackname} This project was supported by the Swiss National Science Foundation [grant number 0003390].
\end{credits}
%
%
%
\bibliographystyle{splncs04}
\bibliography{bib/refs}

\end{document}

%% file: preamble/packages.tex
\usepackage[utf8]{inputenc}

\usepackage[sc]{mathpazo}

\usepackage{tikz}
\usetikzlibrary{positioning, arrows.meta}

\usepackage{amsmath,amssymb,amsfonts,mathrsfs,bm}

\usepackage{graphicx}

\usepackage{xcolor}

\usepackage{mathtools}

\usepackage{microtype}

\usepackage{mleftright}

\usepackage{xspace}

\usepackage[linkcolor=black,colorlinks=true,citecolor=black,filecolor=black]{hyperref}

\usepackage[ruled,vlined,linesnumbered]{algorithm2e}

\usepackage{cleveref}

\makeatletter
\AddToHook{cmd/appendix/before}{\def\cref@section@alias{appendix}}
\makeatother

\usepackage{enumitem}

\usepackage{subcaption}

\usepackage[sort]{cite}

%% file: preamble/macros.tex
\newcommand{\N}{\ensuremath{\mathbb{N}}}

\newcommand{\R}{\ensuremath{\mathbb{R}}}

\DeclarePairedDelimiter{\Paren}{(}{)}
\DeclarePairedDelimiter{\Bracket}{[}{]}
\DeclarePairedDelimiter{\Brace}{\{}{\}}

\NewDocumentCommand{\Oh}{som}{%
  \ensuremath{\mathcal{O}\IfBooleanTF{#1}{%
    \Paren*{#3}
  }{%
    \IfNoValueTF{#2}{%
      \Paren{#3}
    }{%
      \Paren[#2]{#3}
    }%
  }%
}}

\NewDocumentCommand{\Om}{som}{%
\ensuremath{\Omega\IfBooleanTF{#1}{%
    \Paren*{#3}%
  }{%
    \IfNoValueTF{#2}{%
      \Paren{#3}%
    }{%
      \Paren[#2]{#3}%
    }%
  }%
}}

\NewDocumentCommand{\Th}{som}{%
\ensuremath{\Theta\IfBooleanTF{#1}{%
    \Paren*{#3}%
  }{%
    \IfNoValueTF{#2}{%
      \Paren{#3}%
    }{%
      \Paren[#2]{#3}%
    }%
  }%
}}

\NewDocumentCommand{\oh}{som}{%
\ensuremath{o\IfBooleanTF{#1}{%
    \Paren*{#3}%
  }{%
    \IfNoValueTF{#2}{%
      \Paren{#3}%
    }{%
      \Paren[#2]{#3}%
    }%
  }%
}}

\NewDocumentCommand{\om}{som}{%
\ensuremath{\omega\IfBooleanTF{#1}{%
    \Paren*{#3}%
  }{%
    \IfNoValueTF{#2}{%
      \Paren{#3}%
    }{%
      \Paren[#2]{#3}%
    }%
  }%
}}

\NewDocumentCommand{\Prob}{som}{%
\ensuremath{\mathbb{P}
  \IfBooleanTF{#1}{%
    \Bracket*{#3}%
  }{%
    \IfNoValueTF{#2}{%
      \Bracket{#3}%
    }{%
      \Bracket[#2]{#3}%
    }%
  }%
}}

\NewDocumentCommand{\E}{som}{%
\ensuremath{\mathbb{E}
  \IfBooleanTF{#1}{%
    \Bracket*{#3}%
  }{%
    \IfNoValueTF{#2}{%
      \Bracket{#3}%
    }{%
      \Bracket[#2]{#3}%
    }%
  }%
}}

\NewDocumentCommand{\Max}{s o m m}{%
\ensuremath{\max
  \IfBooleanTF{#1}{%
    \Brace*{#3,\, #4}%
  }{%
    \IfNoValueTF{#2}{%
      \Brace{#3,\, #4}%
    }{%
      \Brace[#2]{#3,\, #4}
    }%
  }%
}}

\NewDocumentCommand{\Min}{s o m m}{%
\ensuremath{\min
  \IfBooleanTF{#1}{%
    \Brace*{#3,\, #4}%
  }{%
    \IfNoValueTF{#2}{%
      \Brace{#3,\, #4}%
    }{%
      \Brace[#2]{#3,\, #4}
    }%
  }%
}}

\NewDocumentCommand{\Exp}{som}{%
\ensuremath{\exp
  \IfBooleanTF{#1}{%
    \Paren*{#3}%
  }{%
    \IfNoValueTF{#2}{%
      \Paren{#3}%
    }{%
      \Paren[#2]{#3}%
    }%
  }%
}}

\NewDocumentCommand{\Log}{som}{%
\ensuremath{\log
  \IfBooleanTF{#1}{%
    \Paren*{#3}%
  }{%
    \IfNoValueTF{#2}{%
      \Paren{#3}%
    }{%
      \Paren[#2]{#3}%
    }%
  }%
}}

\NewDocumentCommand{\Var}{som}{%
\ensuremath{\mathrm{Var}
  \IfBooleanTF{#1}{%
    \Bracket*{#3}%
  }{%
    \IfNoValueTF{#2}{%
      \Bracket{#3}%
    }{%
      \Bracket[#2]{#3}%
    }%
  }%
}}

\NewDocumentCommand{\SetBuilder}{s o m m}{%
\ensuremath{\IfBooleanTF{#1}{%
    \Brace*{\,#3 : #4\,}%
  }{%
    \IfNoValueTF{#2}{%
      \Brace{\,#3 : #4\,}%
    }{%
      \Brace[#2]{\,#3 : #4\,}%
    }%
  }%
}}

\newcommand{\binval}{\textsc{BinVal}\xspace}

\renewcommand{\phi}{\ensuremath{\varphi}}

\newcommand{\mpoEA}{\ensuremath{(\mu+1)\text{-EA}}\xspace}



%% file: chapters/00-introduction.tex
\section{Introduction}
\label{sec:introduction}

The runtime analysis of evolutionary algorithms (EAs) studies the expected number of fitness evaluations until an optimal solution is first produced. Over the past two decades, this approach has yielded precise bounds for many algorithm--benchmark combinations and has clarified how design parameters such as the mutation rate or the population size affect performance; see the surveys~\cite{Lengler2020,Sudholt2020} for an overview.

A classical result is that the $(1+1)$~EA with mutation rate~$1/n$ optimizes any linear function $f(x) = \sum_{i \in [n]} w_i x_i$ with positive weights~$w_i$ in expected time $O(n \log n)$~\cite{DrosteJansenWegener2002}. Witt~\cite{Witt2013} later tightened this to $en \ln n + O(n)$. Subsequent work provided alternative proofs and extensions via multiplicative drift, adaptive drift, and related techniques~\cite{DoerrJohannsenWinzen2012,DoerrGoldberg2013,LenglerSteger2018}. %
While single-individual algorithms on linear functions are thus well understood, practical EAs typically maintain \emph{populations} of candidate solutions.

The $(\mu+1)$~EA, which keeps a population of $\mu$ individuals and in each generation creates one offspring by mutating a uniformly chosen parent, is a prototypical population-based elitist algorithm. Understanding the impact of the population size~$\mu$ on the optimization time --- the study of \emph{population dynamics} --- remains one of the most challenging aspects of EA theory~\cite{KrejcaNeumannWitt2025,Sudholt2020}. %
The first rigorous runtime analysis of the $(\mu+1)$~EA was carried out by Witt~\cite{Witt2006}, who showed that the expected runtime on \textsc{OneMax} is $O(\mu n + n\log n)$, a bound later confirmed to be tight~\cite{AntipovDoerr2021}. These results show that for \textsc{OneMax}, the population adds an overhead that is at most linear in~$\mu$. %
However, the effect of population size can be more dramatic in general: Witt~\cite{Witt2008} demonstrated that there exist functions where increasing~$\mu$ by a constant factor reduces the runtime from exponential to polynomial, while conversely, Lengler and Zou~\cite{LenglerZou2021} proved that larger populations can cause superpolynomial slowdowns on certain monotone functions.

The \textsc{BinVal} function, defined as $\textsc{BinVal}(x) = \sum_{i \in [n]} 2^{n-i} x_i$, assigns exponentially decreasing weights to the bit positions. Because an offspring is accepted if and only if it improves the most significant bit at which it differs from the current worst individual, improvements propagate from more significant to less significant positions. Especially with high mutation rates, fixing a significant bit can cause many detrimental mutations of less significant bits. In a population setting, even if all individuals have set a certain bit correctly, this fixation can be destroyed when a more significant bit is improved. Since this effect is most extreme for \textsc{BinVal} among all linear functions, it is a particularly interesting benchmark for studying population effects. While all linear functions have the same asymptotic runtime $O(n \log n)$ for the $(1+1)$~EA, this is no longer the case for population-based algorithms: Doerr and K\"unnemann~\cite{DoerrKuennemann2015} showed that for the $(1+\lambda)$~EA, different linear functions can have different asymptotic runtimes, with \textsc{BinVal} being an asymptotic worst case. 
The \textsc{BinVal} function has also received attention in the dynamic setting~\cite{LenglerRiedi2022,LenglerMeier2020} and in anytime and fixed-target analyses~\cite{KoetzingSander2026}, further establishing it as a key benchmark for population effects on weighted linear functions.

For the $(\mu+1)$~EA on \textsc{BinVal}, generic approaches such as fitness levels and drift-based analyses lead to bounds with an additive $\Theta(n^2)$ term, failing to recover the $O(n\log n)$ scaling even for constant~$\mu$; see the discussion in~\cite{KrejcaNeumannWitt2025}. This was resolved recently by Krejca, Neumann, and Witt~\cite{KrejcaNeumannWitt2025}, who introduced a block-based analysis of population diversity within short contiguous blocks of bits. Their approach yields a runtime bound of $O(\mu^5 n \log(n/\mu^4))$ for standard bit mutation, matching $O(n\log n)$ for constant~$\mu$ and eliminating the quadratic additive term. However, the fifth-degree polynomial dependence on~$\mu$ appears to be an artifact of the proof technique --- the experimental results in~\cite{KrejcaNeumannWitt2025} suggest that the theoretical $\mu^5$ factor is too pessimistic. %
For a more comprehensive account of the related literature, we refer the reader to~\cite{KrejcaNeumannWitt2025} and the references therein, as well as the surveys~\cite{Lengler2020,Sudholt2020}.

\subsection{Our Contribution}

In this work, we substantially improve the upper bound on the expected runtime of the $(\mu+1)$~EA on \textsc{BinVal}. Recall that the $(\mu+1)$~EA maintains a population of~$\mu$ individuals; in each iteration it selects a parent uniformly at random, creates an offspring by mutation, and retains the best~$\mu$ individuals (see Sections~\ref{sec:preliminaries} and~\ref{sec:standardmutation} for the precise definitions). We show runtime results for a generic class of mutation operators, which we call \emph{conservative operators}. Those include in particular \emph{0/1-bit mutation}, which flips exactly one uniformly chosen bit with probability~$1-p$ and leaves the string unchanged with probability~$p$ for a constant $p \in (0,1)$; and \emph{standard bit mutation}, which flips each bit independently with probability~$\chi/n$ for any $\chi >0$. Then our main result is the following.

\begin{theorem}\label{thm:main}
Consider the $(\mu+1)$~EA on the \textnormal{\textsc{BinVal}} function of dimension~$n$, with population size $\mu \geq 2$. Under any conservative mutation operator or specifically either 0/1-bit mutation with constant cloning probability~$p \in (0,1)$, or standard bit mutation with rate~$\chi/n$ for constant $\chi > 0$, the expected runtime is
\[
  O\!\left(\mu \log \mu \cdot n \log n\right).
\]
\end{theorem}
More precisely, the expected runtime is $O(\mu \log \mu \cdot n \max\{1,\log(n/(\mu \log \mu))\})$, which is a stronger bound when $\mu$ grows with~$n$. This reduces the dependence on the population size from $\mu^5$ in the bound of~\cite{KrejcaNeumannWitt2025} to $\mu \log \mu$, while also substantially relaxing the restriction on the range of admissible~$\mu$: the result of~\cite{KrejcaNeumannWitt2025} requires $\mu = o((n/\log n)^{1/5})$, while our result improves the previous best bounds for any $\mu = o(n/\log n)$. Our upper bound also closes the gap to the known $\Omega(\mu n + n \log n)$ lower bound~\cite{Witt2006} up to logarithmic factors. For constant~$\mu$, the bound recovers the $O(n \log n)$ runtime of the $(1+1)$~EA on linear functions~\cite{Witt2013}. Finally, while~\cite{KrejcaNeumannWitt2025} only considered mutation rate $1/n$ for standard bit mutation, our result extends to any mutation rate $\chi/n$ with $\chi = O(\sqrt n)\cap\Omega(1)$ at the cost of a factor $e^{\chi}$.

Our improvement stems from a more refined analysis of the \emph{interference probability} during the spreading of a beneficial mutation through the population. In the framework of~\cite{KrejcaNeumannWitt2025}, the bit string is partitioned into contiguous blocks, and within each block, mutations spread from a single copy to all~$\mu$ individuals. During this spreading process, an interfering mutation at another position in the same block can disrupt this fixation. While~\cite{KrejcaNeumannWitt2025} bounds this interference probability uniformly by the block size, we observe that it depends on the number~$k$ of remaining unfixed zeros in the current block. When few zeros remain, interference is less likely and mutations spread more easily. This $k$-dependent analysis, combined with Markov chain arguments for modeling population fixation and additive drift analysis for bounding macro-step progress, yields the improved bound for mutation rate at most $1/n$. 

In order to extend the result to non-constant mutation rates, we show that every run of the $(\mu+1)$ EA with mutation rate $\chi/n$ on \binval with $\chi >1$ can be decomposed into a sequence of $\lceil\chi\rceil$ optimization processes of \binval with smaller dimension and smaller mutation parameter $\chi' \le 1$. Although a rather basic observation, we believe that this decomposition is of general interest, as it generalizes far beyond the \mpoEA. 

Finally, we present experimental results for the \mpoEA on \binval for both standard bit mutation and 0/1 mutation to provide empirical validation and to investigate the tightness of the analysis.


%% file: chapters/01-preliminaries.tex
\section{Preliminaries}
\label{sec:preliminaries}

\subsubsection{Notation and Problem Setting.}

The \mpoEA maintains a population $P_t$ of size $\mu$ at every time step $t$ for some $\mu \in \N$. The objective is to optimise a fitness function $f \colon \{0, 1\}^n \to \R$. In each generation, a uniformly random individual $x$ is chosen from the population, and an offspring $y$ is produced from $x$ through mutation. The offspring $y$ is then added to the population, and the individual with the worst fitness (possibly $y$ itself) is discarded. In this paper, we study the \mpoEA for $\mu \geq 2$ on the \binval fitness function, defined for $x \in \{0,1\}^n$ as
\begin{equation*}
    \binval(x) = \sum\nolimits_{i=1}^n 2^{n-i}\cdot x_i.
\end{equation*}
Because \binval is injective, every distinct bit string yields a unique fitness value, ensuring that the removal of the worst individual is always unambiguous.

We analyse a class of \emph{conservative mutation operators} that is defined in Section~\ref{sec:generalisation}. It includes in particular \textit{$0/1$ mutation}, in which the offspring $y$ is an exact clone of $x$ with probability $p \in (0, 1)$, and otherwise, a single bit of $x$ is chosen uniformly at random and flipped to produce $y$. And it includes \textit{standard bit mutation}, in which every bit of $x$ is flipped independently with probability $\chi/n$ to produce $y$, for some \emph{mutation parameter} $0<\chi \le 1$, which we later extend to $\chi >1$ as well.

In our analysis, we partition each bit string into $\nu := n/n'$ blocks $B_i$ of length $n'$, for $i \in [\nu]$. We denote by $B_i(x)$ the substring corresponding to the $i$-th block of bits in $x$, specifically $B_i(x) := x_{(i-1)n'+1}\dots x_{i\cdot n'}$. We call a position $j$ \textit{fixated} at time $t$ if $\forall x \in P_t \colon x_j = 1$. Note that once a leading segment is fixated for \binval, it remains fixated forever. We are interested in the expected time it takes for all positions in the current block to become fixated, conditional on finishing the previous block. Formally, we define $S_0 = 0$ and let $S_i$ be the first time step where all bits in block $B_i$ are fixated across the entire population:
\begin{equation*}
    S_i := \inf \{ t \in \N \mid t \geq S_{i-1} \wedge \forall x \in P_t \colon B_i(x)=1^{n'} \}.
\end{equation*}
We are interested in $\mathbb{E}[T_i]$, where $T_i := S_i - S_{i-1}$. For convenience, at time step $t$, we denote by $B$ the \textit{current block} $B_i$, where $i \in [\nu]$ such that $S_{i-1} \leq t < S_i$. The total time $T$ for the process to produce an optimal individual is given by $T = \sum_{i=1}^{\nu} T_i$, and by linearity of expectation, $\mathbb{E}[T] = \sum_{i=1}^\nu \mathbb{E}[T_i]$.

At time $t$, we say the process is in a \textit{flip phase} if $\forall x,y \in P_t \colon B(x) = B(y)$; otherwise, we say it is in a \textit{spread phase}. Note that the process may only switch from a flip phase to a spread phase when there is at least one $0 \rightsquigarrow 1$ mutation in the current block $B$. Let $q$ denote the position of the most significant non-fixated $1$ in the population. We refer to position $q$ as the \textit{current mutation} with respect to a spread phase. During a spread phase, two failure events may occur:
\begin{itemize}
    \item \textit{Interference:} A $0 \rightsquigarrow 1$ mutation is accepted in an unfixated position $b \in B$ which is more significant than $q$, that is, $b < q$.
    \item \textit{Destruction:} A previously fixated position becomes unfixated (flip $1 \rightsquigarrow 0$).
\end{itemize}

Our convention is that upon both interference and destruction a new spread phase starts. If, during a spread phase, position $q$ becomes fixated, the process either switches back to a flip phase or, if there remains another non-fixated $1$ in the population, a new spread phase starts. We refer to a \textit{macro step} as the period from the end of a spread phase until the end of the subsequent spread phase. Hence, a macro step may consist of a flip phase and a spread phase, or possibly only of a spread phase. Intuitively, a macro step represents one continuous attempt to fixate a bit in the current block, where the attempt may fail due to interference or destruction.

\begin{note}
    In the proof, we will later choose a block length of $n' = cn/(\mu\ln \mu) $, where $c$ is a constant depending on the mutation operator. We will see that we may restrict to the case $\mu = o(n/\log n)$ in our proofs, which implies $n' \ge 1$ for sufficiently large $n$. For the sake of exposition, we will assume that $n'$ is an integer. Otherwise, we may round down to the next integer. 
    Similarly, it may be the case that $n'$ does not divide $n$, forcing the final block to be shorter, and for the sake of exposition we assume that $n'$ divides~$n$.
\end{note}

\subsubsection{Tools.}
In our analysis we will employ the following additive drift theorem:~\cite{additive-drift}
\begin{theorem}
    Let $(X_t)_{t \geq 0}$ be a sequence of non-negative random variables with a finite state space $\mathcal{S} \subseteq \mathbb R_{\geq0}$ such that $0 \in \mathcal{S}$. Let $T := \inf\{t \geq 0\mid X_t = 0\}$. If there exists $\delta > 0$ such that for all $s \in \mathcal S\setminus \{0\}$ and for all $t \geq 0$
    \[
    \Delta_t(s) := \mathbb E[X_t - X_{t+1} \mid X_t = s] \geq \delta,
    \]
    then
    $\mathbb E[T] \leq \mathbb E[X_0]/\delta$.
    \label{thm:additive-drift}
\end{theorem}

We will also use the following basic observation.
\begin{lemma}\label{lem:mulogmu}
Assume $\mu= o(n/\log n)$. Then $\mu \log \mu = o(n)$.
\end{lemma}
\begin{proof}
    For sufficiently large $n$ we have $\mu \le n$ and thus $\log \mu \le \log n$. Hence, $\mu \log \mu = o(n/\log n \cdot \log n) = o(n)$.~\qed
\end{proof}

%% file: chapters/02-generalisation.tex
\section{Proofs}

\subsection{Conservative Mutation Operators}\label{sec:generalisation}


In this section we will introduce the class of conservative mutation operators. Then we will show the main theorem, the runtime bound for the \mpoEA with arbitrary conservative mutation operator.

\begin{definition}[Conservative Mutation Operator]
\label{def:conservative_mutation}
A mutation operator is said to be \emph{conservative} if there exist constants $c_1, c_2 > 0$ such that for all disjoint sets of bits $S,S' \subseteq \{1, \dots, n\}$, the following conditions are satisfied:
\begin{itemize}
    \item The probability of producing an exact clone, $p_{\text{clone}}$, is $\Omega(1)$.
    
    \item The conditional probability $p_{\text{safe}}$ that in a single iteration only bits in $S$ are flipped given that some bit in $S$ is flipped, is $\Omega(1)$.
    
    \item The probability $p_{\text{flip,safe}}$ of flipping at least one bit in $S$ and no bit in $[n] \setminus S$ in a single iteration satisfies
    $p_{\text{flip,safe}} \ge c_1|S|/n$.

    
    \item 
    The expected number of bits in $S$ that are flipped in a single iteration is at most $\mathbb{E}[\text{destroy}] \le c_2 |S|/n$. The bound remains true if we condition on flipping at least one bit in $S'$ in that iteration. 
\end{itemize}
\end{definition}

Given these conditions, we can systematically bound the expected optimization time. We first analyze the probability of a mutation successfully taking over the population.

\begin{lemma}[Spread Probability]
    \label{lem:spread_prob}
    During a spread phase, the probability that the current mutation successfully fixates across the population before a failure event occurs is at least $(e\mu)^{-C}$, where $C = c_2n'\mu/(np_{\text{clone}})$.
\end{lemma}
\begin{proof}
    We model the spread of the current mutation as a Markov Chain with transient states $i \in \{1,2,\dots, \mu -1\}$, a target absorbing state $\mu$ (successful fixation), and a failure absorbing state $F$. State $i$ represents $i$ individuals currently holding the mutation. The transition probabilities are bounded by:
    \begin{itemize}
        \item \textbf{Progress:} $p_{i,i+1} \geq p_{\text{clone}} \cdot i/\mu$ (the probability of cloning an individual with the current mutation).
        \item \textbf{Failure:} $p_{i,F} \leq c_2 n'/n$ (since the probability of flipping a bit in the current block is upper bounded by $\mathbb{E}[\text{destroy}]$ by Markov's inequality).
        \item \textbf{Staying Put:} $p_{i,i} = 1 - p_{i,i+1} - p_{i,F}$.
    \end{itemize}
    Let $s_i$ be the probability of reaching state $\mu$ from state $i$, with $s_\mu = 1$ and $s_F = 0$. By the law of total probability,
    \begin{equation*}
        s_i = p_{i,i+1}s_{i+1} + (1- p_{i,i+1}-p_{i,F})s_i.
    \end{equation*}
    Rearranging yields $s_i = \frac{p_{i,i+1}}{p_{i,i+1}+p_{i,F}}s_{i+1}$. Clearly, this relation implies that $s_1$ is smallest of the probabilities. Thus, it suffices to bound $s_1$. We first expand $s_1$ via
    \begin{equation*}
        s_1 = \prod^{\mu-1}_{i=1}\frac{p_{i,i+1}}{p_{i,i+1}+p_{i,F}} = \prod^{\mu-1}_{i=1}\frac{1}{1 + \frac{p_{i,F}}{p_{i,i+1}}}.
    \end{equation*}
    Using the bound $\frac{p_{i,F}}{p_{i,i+1}} \leq \frac{c_2n'\mu}{np_{\text{clone}}i} = \frac{C}{i}$ and the standard inequality $\frac{1}{1+x} \geq e^{-x}$,
    \begin{equation*}
        s_1 \geq \prod^{\mu-1}_{i=1}e^{-C/i} = e^{-C\sum^{\mu-1}_{i=1}1/i} = e^{-CH_{\mu-1}}.
    \end{equation*}
    As $H_{\mu-1} \leq \ln \mu + 1$, the probability of fixating the current mutation is at least
    \begin{equation*}
        \mathbb P[\text{fixate}] \geq e^{-C(\ln \mu + 1)} = e^{-C}\cdot \mu^{-C} = (e\mu)^{-C}.\tag*{\qed}
    \end{equation*}
\end{proof}

\begin{lemma}[Phase Durations]
    \label{lem:phase_durations}
    Let $k$ be the number of non-fixated bits in the current block. The expected duration of a flip phase is $\mathbb{E}[T^{k}_{\text{flip}}] \leq n/(c_1k)$, and the expected duration of a spread phase is $\mathbb{E}[T_{\text{spread}}] \leq \mu (\ln \mu + 1)/p_{\text{clone}}$.
\end{lemma}
\begin{proof}
    Taking $S$ to be the set of $k$ non-fixated bits in the current block, $\mathbb{E}[T^{k}_{\text{flip}}]$ is stochastically dominated by a geometric distribution with success probability at least $c_1k/n$, yielding $\mathbb{E}[T^{k}_{\text{flip}}] \leq n/(c_1k)$. 
    
    To bound the expected time of a spread phase, we assume the worst-case scenario where no failure events truncate the phase, relying purely on the progress probability $p_{i,i+1}$:
    \begin{equation*}
        \mathbb E[T_{\text{spread}}] \leq \sum^{\mu-1}_{i=1}\frac{1}{p_{i,i+1}} \leq \sum^{\mu-1}_{i=1}\frac{\mu}{p_{\text{clone}}\cdot i} = \frac{\mu}{p_{\text{clone}}}H_{\mu -1} \leq \frac{\mu (\ln \mu + 1)}{p_{\text{clone}}}.\tag*{\qed}
    \end{equation*}
\end{proof}

We now evaluate the process in terms of macro steps. Let $X_m$ denote the number of non-fixated bits in the current block at macro step $m$. To analyze the time it takes to reduce the non-fixated bits from $k$ to $k-1$, we shift $X_m$ to define a zero-state: $X^{k-1}_m = X_m - (k-1)$.

\begin{lemma}[Macro Step Drift]
    \label{lem:macro_drift}
    Assume $\mu=o(n/\log n)$. For a suitable constant $c>0$ and block size $n' := cnp_{\text{clone}}/(c_2\mu \ln \mu)$, the expected drift $\Delta^k_m(s) := \mathbb{E}[X^{k-1}_m - X^{k-1}_{m+1} \mid X^{k-1}_m = s - (k-1)]$ is $\Delta^k_m(s) = \Omega(1)$.
\end{lemma}
\begin{proof}
    The drift depends on the probability $P_{s, s-1}$ of successfully fixating a bit, and $P_{s,s+i}$ of unfixating $i$ previously fixated bits:
    \begin{equation*}
        \Delta^k_m(s) = P_{s, s-1}-\sum^{n'-s}_{i=1} i\cdot P_{s,s+i}.
    \end{equation*}
    We establish bounds for these terms:
    \begin{itemize}
        \item \textbf{Progress:} If the macro step starts in a flip phase, then consider the set $S$ of unfixated bits in $B$. When one of them is flipped, the probability of not flipping any bit outside of $S$ is at least $p_{\text{safe}}$. In this case a spread phase is entered, and the probability of successfully fixating the current mutation is at least $(e\mu)^{-C}$ by Lemma~\ref{lem:spread_prob}. Hence,
        $P_{s, s-1} \geq   (e\mu)^{-C} \cdot p_{\text{safe}}$.
        \item \textbf{Destruction:} A macro step consists of at most one flip phase and one spread phase, and destruction can only occur in the spread phase or the final iteration of a flip phase. In the latter, if we define $S'$ to be the set of all non-fixated positions in $B$ then we condition on at least one bit flip in $S'$. The same bound on the expected number of bit flips applies to both cases:
        \begin{equation*}
            \sum^{n'-s}_{i=1} i\cdot P_{s,s+i} \leq \mathbb{E}[\text{destroy}] \cdot \left(\mathbb E[T_{\text{spread}}]+1\right) \leq c_2\frac{n'}{n}\cdot\left(\frac{\mu (\ln \mu + 1)}{p_{\text{clone}}} + 1\right).
        \end{equation*}
    \end{itemize}
    Combining these bounds yields
    \begin{equation*}
        \Delta^k_m(s) \geq (e\mu)^{-C} \cdot p_{\text{safe}} - c_2\frac{n'}{n}\cdot\left(\frac{\mu (\ln \mu + 1)}{p_{\text{clone}}} + 1\right).
    \end{equation*}
    Note that $n' \ge 1$ by Lemma~\ref{lem:mulogmu}, and recall from Lemma~\ref{lem:spread_prob} that $C = \frac{c_2n'\mu}{np_{\text{clone}}} = \frac{c}{\ln \mu}$. Substituting this and $n'$ yields
    \begin{align*}
        \Delta^k_m(s) &\geq e^{-c(1+1/\ln\mu)} \cdot p_{\text{safe}} - c\cdot\frac{p_{\text{clone}}}{\mu \ln \mu}\cdot\left(\frac{\mu (\ln \mu + 1)}{p_{\text{clone}}} + 1\right)\\
        &\geq p_{\text{safe}}\left(1-c\left(1+\frac{1}{\ln \mu}\right)\right)- c\cdot \frac{(\ln \mu + 1)}{\ln \mu} - c\cdot \frac{p_{\text{clone}}}{ \mu \ln \mu}.
    \end{align*}
    The right hand side is of the form $p_{\text{safe}}-c\cdot O(1)$. Thus if the constant $c > 0$ is sufficiently small then the drift is $\Omega(1)$.~\qed
\end{proof}

\begin{theorem}
    \label{thm:main_bound}
    The expected optimization time of the \mpoEA on \binval with a conservative mutation operator is ${O}\left(\mu \log \mu \cdot n \log n\right)$ for $\mu = o(n / \log n)$, and ${O}\left(\mu \log \mu \cdot n\right)$ for $\mu = \Omega(n / \log n)$.
\end{theorem}
\begin{proof}
    First consider $\mu = o(n/\log n)$. Let $\mathbb{E}[T_i]$ be the expected time to fixate block $i$. Let $T_{i,k}$ be the first point in time when there are $k$ fixated positions in block $i$. By Lemma \ref{lem:macro_drift} and the additive drift theorem, Theorem \ref{thm:additive-drift}, we need in expectation $O(1)$ macro steps to go from time $T_{i,k}$ to time $T_{i,k-1}$. Note that the number of non-fixated positions is not necessarily decreasing, but the additive drift theorem also applies to non-monotone processes. While $T_{i,k-1}$ is not reached, each flip phase has $k'$ non-fixated bits for some $k'\ge k$. Hence, by Lemma~\ref{lem:phase_durations} its expected duration is  $\mathbb E[T_{\text{flip}}^{k'}] \le n/(c_1k') \le n/(c_1k)$. Together with the bound $\mathbb{E}[T_{\text{spread}}] \leq \mu (\ln \mu + 1)/p_{\text{clone}}$ for a spread phase, and since each macro step consists of at most one flip phase and one spread phase, we obtain  
    \begin{align*}
        \mathbb E[T_i] & = \sum^{n'}_{k=1} \mathbb E[T_{i,k-1}-T_{i,{k}}] 
        \leq {O}(1) \cdot\sum^{n'}_{k=1}\left(\frac{n}{c_1k} + \frac{\mu (\ln \mu+1)}{p_{\text{clone}}}\right) \\
        &\leq {O}(1) \cdot \left(\frac{n}{c_1}(\ln n'+1) + n'\cdot \frac{\mu (\ln \mu+1)}{p_{\text{clone}}}\right).
    \end{align*}
    The total time is bounded by $\mathbb E[T] = \sum_{i=1}^\nu\mathbb E[T_i] = \frac{n}{n'} \mathbb{E}[T_i]$. This gives
    \begin{align*}
        \mathbb E[T] &\leq {O}(1) \cdot \left(\frac{n^2}{c_1n'}(\ln n'+1) + \frac{n\mu (\ln \mu+1)}{p_{\text{clone}}}\right).
    \end{align*}
    Now we substitute our chosen block size $n' = \Theta(\tfrac{n}{\mu \log \mu})$ into the first term:
    \begin{align*}
        \mathbb E[T] &= {O}(1) \cdot \left(\frac{c_2 n \mu \ln \mu}{c_1 c p_{\text{clone}}}(\ln n' + 1) + \frac{n\mu (\ln \mu+1)}{p_{\text{clone}}}\right) 
        = {O}\left(\mu \log \mu \cdot n \log n'\right).
    \end{align*}
    Since $n' \leq n$, we obtain the final bound $\mathbb E[T] = {O}\left(\mu \log \mu \cdot n \log n\right)$.
    
    For the case $\mu = \Omega(n / \log n)$, as already noted in~\cite{KrejcaNeumannWitt2025}, we can use blocks of length $n'=1$, in which case no interference or destruction can occur. We then easily obtain a bound of $ O(n\mu\log\mu + n^2)$, since each flip phase takes time $n$ and each spread phase takes time $ O(\mu\log \mu)$. The term $n^2$ can be omitted because $\mu\log \mu = \Omega(n)$.~\qed
\end{proof}

\begin{note}
The proof implies that the bound in Theorem~\ref{thm:main_bound} also holds for the time until \emph{all} individuals in the population are at the optimum.
\end{note}

%% file: chapters/03-zeroonemutation.tex
\subsection{$0/1$ Mutation}
\label{sec:zeroonemutation}

The $0/1$ mutation operator is defined as the stochastic choice between identity and a single-bit flip. Specifically, it produces a clone of the parent with a constant probability $p \in (0,1)$, and performs an RLS mutation, flipping exactly one uniformly chosen bit, with probability $1-p$.

\begin{corollary}
The $0/1$ mutation operator is a conservative mutation operator. Consequently, the expected optimization time of the \mpoEA on \binval using $0/1$ mutation is ${O}(\mu \log \mu \cdot n \log n)$.
\end{corollary}
\begin{proof}
We verify the conditions of Definition~\ref{def:conservative_mutation} to apply Theorem~\ref{thm:main_bound}.

By definition, $p_{\text{clone}} = p =\Omega(1)$. At most one bit is flipped per invocation, hence we trivially get $p_{\text{safe}} = 1$. With probability $1-p$, a single bit is chosen uniformly at random and flipped. Therefore, also using the previous insight,
\begin{equation*}
    p_{\text{flip,safe}}  = \mathbb E[\text{destroy}] = (1-p)|S|/n,
    \label{eq:zero-one-pflip}
\end{equation*}
resulting in the desired bounds for $p_{\text{flip,safe}}$ and $\mathbb E[\text{destroy}]$ with constants $c_1 = c_2 = 1-p$. The bound on $\mathbb E[\text{destroy}]$ also holds when conditioning on flipping a bit in $S'$ as at most 1 bit is flipped per iteration.

As all conditions for a conservative mutation operator are met, the complexity bound follows from Theorem~\ref{thm:main_bound}.~\qed
\end{proof}

%% file: chapters/04-standardmutation.tex
\subsection{Standard Bit Mutation with $\chi \le 1$}
\label{sec:standardmutation}

The standard bit mutation operator modifies an individual by flipping each bit independently with probability $p = \chi/n$, for some mutation rate $\chi > 0$.

\begin{corollary}\label{cor:SBM_large_chi}
    For any constant mutation parameter $\chi \leq 1$, the standard bit mutation operator with rate $\chi/n$ is a conservative mutation operator. Consequently, the expected optimization time of the \mpoEA on \binval under standard bit mutation is ${O}(\mu \log \mu \cdot n \log n)$.
\end{corollary}

\begin{proof}
We verify the conditions of Definition~\ref{def:conservative_mutation} to invoke Theorem~\ref{thm:main_bound}.

\begin{itemize}
    \item $p_{\text{clone}}$: A clone is produced if no bits are flipped, yielding $p_{\text{clone}} = (1-\chi/n)^n = \Theta(1)$. 
    \item $p_{\text{safe}}$: For a subset of bits $S \subseteq [n]$, a mutation is considered safe if no bits in $[n] \setminus S$ are modified. Due to the independence of bit flips, we have $p_{\text{safe}} = (1-\chi/n)^{n - |S|} = \Theta(1)$.
    
    \item $p_{\text{flip,safe}}$: The probability $p_{\text{flip,safe}}$ of flipping at least one bit in $S$ while flipping no bits in $[n] \setminus S$ is given by
    \[ 
    p_{\text{flip,safe}} = \big(1 - \left(1 - \chi/n\right)^{|S|}\big) \cdot \left(1 - \chi/n\right)^{n - |S|}.
    \]
    We use the inequality 
    \[
    1-(1-x)^k = x\cdot \sum^{k-1}_{i=0}(1-x)^i \geq kx(1-x)^{k-1},
    \]
    for $x \in (0,1)$, derived from the geometric sum formula, to lower bound the first term by $\frac{\chi|S|}{n}(1-\frac{\chi}{n})^{|S|-1}$. Multiplying by the second term results in the lower bound $\frac{\chi|S|}{n}(1-\frac{\chi}{n})^{n-1} \geq \chi e^{-\chi} \frac{|S|}{n}$ for $\chi \leq 1$, hence there exists a constant $c_1$ such that $p_{\text{flip,safe}} \ge c_1|S|/n$.


    \item $\mathbb{E}\text{[destroy]}$: The number of bits flipped within $S$ follows a binomial distribution $B(|S|, \chi/n)$. The expected number of such flips is $\mathbb{E}[\text{destroy}] = \chi|S|/n$. This satisfies the condition $\mathbb{E}[\text{destroy}] \le c_2 \frac{|S|}{n}$ for a constant $c_2 \ge \chi$. When conditioning on flipping at least one bit in $S'$ the same bound holds as bits are flipped independently.
\end{itemize}

Since the operator is conservative, the result follows from Theorem~\ref{thm:main_bound}.~\qed
\end{proof}

\subsection{Standard Bit Mutation with $\chi \ge 1$}
\label{sec:generalchi}

To bound the runtime for larger mutation parameters $1 < \chi < n$, we couple the \mpoEA on \binval with mutation rate $\chi/n$ to $\lceil \chi\rceil$ sequential optimization processes with mutation rate at most $1/n$. 

Formally, we use the following proposition. For a (possibly randomized) algorithm $A$ let the \emph{lazy} version of $A$ with idle probability $p_{\text{idle}}$ be the algorithm which flips a coin in each step: with probability $p_{\text{idle}}$ it does nothing, otherwise it performs the next step of $A$. So the trajectory of the idle version of $A$ is the same as for $A$, but the idle version follows the trajectory at lower speed.

\begin{proposition}\label{prop:reduction_of_chi}
    Let $0 < \chi < n$. Consider a run of the \mpoEA with mutation rate $\chi/n$ on $n$-dimensional \binval. Let $B = \{n_b,n_b+1,\ldots,n_b+n'-1\}$ be a consecutive block of length $n'$, and assume that for any position $b < n_b$, $\forall x \in P_0: x_b = 1$. Let $B(P_t)$ be restriction of the population $P_t$ at time $t$ to $B$.
    
    Now consider the lazy \mpoEA on $n'$-dimensional \binval with $p_{\text{idle}} = 1-(1-\chi/n)^{n_b-1}$ and with mutation rate $\chi'/n'$, where $\chi' = \chi n'/n$. Let $P_t'$ be the population of this lazy version after $t$ steps, and assume $P_0' = B(P_0)$. 
    
    Then $B(P_t)$ and $P_t'$ follow the same distribution for all $t\ge 0$, i.e., for all populations $P \in\{0,1\}^{n'}$ we have $\Pr[B(P_t) = P] = \Pr[P_t' = P]$. 

    In particular, let $T := \inf_t \{\forall x \in B(P_t):x = 1^{n'}\}$ and $T' := \inf_t\{\forall x \in P'_t:x = 1^{n'}\}$. Then $T$ and $T'$ follow the same distribution.
    
\end{proposition}
\begin{proof}
    We prove the statement by induction over $t$. It is true for $t=0$ since $P_0' = B(P_0)$, so assume that as induction hypothesis that it is true for some $t\ge 0$. We need to show that both algorithm perform equivalent update steps, i.e., for all $P_1,P_2\in \{0,1\}^{n'}$ we have 
    \begin{align}\label{eq:induction_step}
        \Pr[B(P_{t+1})= P_2 \mid B(P_t) = P_1] = \Pr[P'_{t+1}= P_2 \mid P_t' = P_1],
    \end{align}
    because then for all $P_2$,
    \begin{align*}
        \Pr[B(P_{t+1}) = P_2] & = \sum_{P_1\in \{0,1\}^{n'}} \Pr[B(P_{t+1})= P_2 \mid B(P_t) = P_1]\cdot \Pr[B(P_t) = P_1] \\
        & = \sum_{P_1\in \{0,1\}^{n'}} \Pr[P'_{t+1}= P_2 \mid P_t' = P_1] \cdot \Pr[P_t' = P_1] = \Pr[P'_{t+1} = P_2].
    \end{align*}
    To show~\eqref{eq:induction_step}, consider one update step of the \mpoEA. With probability $1-(1-\chi/n)^{n_b-1}$ at least one of the first $n_b-1$ bits is flipped. In this case the offspring is rejected and the population stays the same. This corresponds to the probability $p_{\text{idle}}$ of the lazy \mpoEA. If none of the first $n_b-1$ bits is flipped, then each bit in $B$ is flipped with probability $\chi/n = \chi'/n'$, as for the lazy \mpoEA. The selection is based on the bits in $B$, since the whole population agrees on the more significant bits. Note that in case of ties within $B$, selection for the \mpoEA may be based on the less significant bits. However, this only breaks ties between individuals that are identical on $B$, hence $B(P_{t+1})$ is independent of which copy is chosen. Overall, mutation and selection are identical for the \mpoEA and the lazy \mpoEA, which proves~\eqref{eq:induction_step}.

    Finally, the statement on $T$ and $T'$ follows, because $T$ and $T'$ are defined in terms of $B(P_t)$ and $P_t'$, which follow the same distribution.~\qed
\end{proof}

Although we have only formulated Proposition~\ref{prop:reduction_of_chi} for the \mpoEA, the principle is applicable to other algorithms as well. This includes the $(\mu+\lambda)$ EA, but also genetic algorithms with uniform crossover (in which case idle steps only occur for mutation, not for crossover). As we will see below, it gives a very generic argument for ruling out thresholds where the runtime jumps to a different asymptotic level when $\chi$ is increased by an additive constant. In other situations, such jumps do occur frequently~\cite{DoerrJSWZ2013,LissovoiWitt2016,Lengler2019monotone,LenglerSchaller2019,AntipovDoerrYang2019}.

Proposition~\ref{prop:reduction_of_chi} is a rather basic observation, but it has a surprising consequences. Consider the \mpoEA with large mutation parameter $\chi$ on \binval. Then the time needed to optimize a block is coupled to a lazy optimization process of \binval for a smaller value of $n'$, and crucially, for a smaller mutation parameter $\chi'$. As an immediate consequence, we get that the runtime of the \mpoEA with mutation rate $\chi/n$ for $\chi > 1$ can be dominated by the sum of $\lceil\chi\rceil$ runtimes of a lazy \mpoEA with mutation parameter $\chi \in [1/2,1]$. 

\begin{theorem}\label{thm:coupling}
    Let $T_{\chi,n,\max}$ be the worst-case expected time, taken over all initial populations, until the \mpoEA with mutation rate $\chi/n$ on n-dimensional \binval creates a population that consists of copies of the optimum. Let $T_{[1/2,1],n,\max} := \sup_{\chi\in [1/2,1]} T_{\chi,n,\max}$. Then for all $1 \le \chi < n$ and all $n\in\mathbb{N}$,
\begin{align*}
    T_{\chi,n,\max} \le \big(1-\chi/n\big)^{-n}\cdot \lceil\chi \rceil \cdot T_{[1/2,1],\lceil n/\lceil\chi \rceil\rceil,\max}.
\end{align*}
\end{theorem}
\begin{proof}
    We split the string of length $n$ into $\lceil\chi \rceil$ blocks of length $n' = \lceil n/\lceil\chi \rceil \rceil$, where we allow the last block to be shifted so that it overlaps with the penultimate block if $n'$ does not divide $n$. By Proposition~\ref{prop:reduction_of_chi}, after the $i$-th block is optimized, the time until the next block is optimized is bounded by the time until a lazy \mpoEA has optimized an $n'$-dimensional \binval instance. The mutation parameter of the lazy \mpoEA is $\chi n'/n \in [1/2,1]$. The idle steps increase the expected time by a factor $1/(1-p_{\text{idle}})$, where 
    $1-p_{\text{idle}} \ge (1-\chi/n)^{n}$.\qed
\end{proof}

Since Corollary~\ref{cor:SBM_large_chi} gives us a bound on $T_{[1/2,1],\lceil n/\lceil\chi \rceil\rceil,\max} = O(\mu \log \mu \cdot (n/\chi)\cdot \log n)$, we immediately get a runtime bound for large values of $\chi$.

\begin{corollary}
    For any mutation rate $1 \le \chi < n$, the expected optimization time of the \mpoEA  with standard bit mutation and mutation rate $\chi/n$ on \binval is ${O}((1-\chi/n)^{-n} \cdot \mu \log \mu \cdot n \log n )$.

    If $\chi = O(\sqrt{n})$ then this runtime is $O(e^{\chi} \cdot \mu \log \mu \cdot n \log n)$.
\end{corollary}
\begin{proof}
The first formula is immediately implied by Theorem~\ref{thm:coupling} and the bound     $T_{[1/2,1],\lceil n/\lceil\chi \rceil\rceil,\max} = O(\mu \log \mu \cdot (n/\chi)\cdot \log (n/\chi))$ from Corollary~\ref{cor:SBM_large_chi}, where we upper bound $\log(n/\chi) \le \log(n)$. 

For the second formula, the Taylor expansion of the function $\ln(1-x) = -x-O(x^2)$ for $x\to 0$ implies by exponentiating that $1-x = e^{-x-O(x^2)}$. We plug in $x=\chi/n$ and raise both sides to power $-n$, which gives for $\chi=O(\sqrt{n})$,
\begin{align*}\tag*{\qed}
    (1-\chi/n)^{-n} = (e^{-\chi/n - O((\chi/n)^2)})^{-n} = e^{\chi +O(\chi^2/n)} = e^{\chi +O(1)} 
    = O(e^{\chi}).
\end{align*}
\end{proof}

%% file: chapters/05-experiments.tex
\section{Experimental Results}
\label{sec:simulations}
We implemented the \mpoEA using lexicographical comparisons, which allowed us to avoid the large coefficients of the \binval fitness function.
This allows us to measure the runtime across for large problem sizes (specifically $n=1000$). For each parameter configuration, the simulation is initialised uniformly at random and run until the optimum (the all 1s string) is sampled. The reported runtime is averaged over 50 independent trials.

\subsubsection{Standard Bit Mutation Dynamics.}

To evaluate the predictive accuracy of our theoretical models, we first simulate the expected runtime using the standard bit mutation operator, where each bit is flipped independently with probability $\chi/n$. We tested three different mutation parameters: $\chi \in \{0.5, 1.0, 1.5\}$. 

To clearly observe the penalty induced by population size $\mu$, we normalise the empirical expected runtime by $\mu \cdot n$. The results are shown in Figure~\ref{fig:empirical_runtimes} (a).

\begin{figure}[htbp]
    \centering
\begin{subfigure}{0.4\textwidth}
\includegraphics[width=0.95\textwidth]{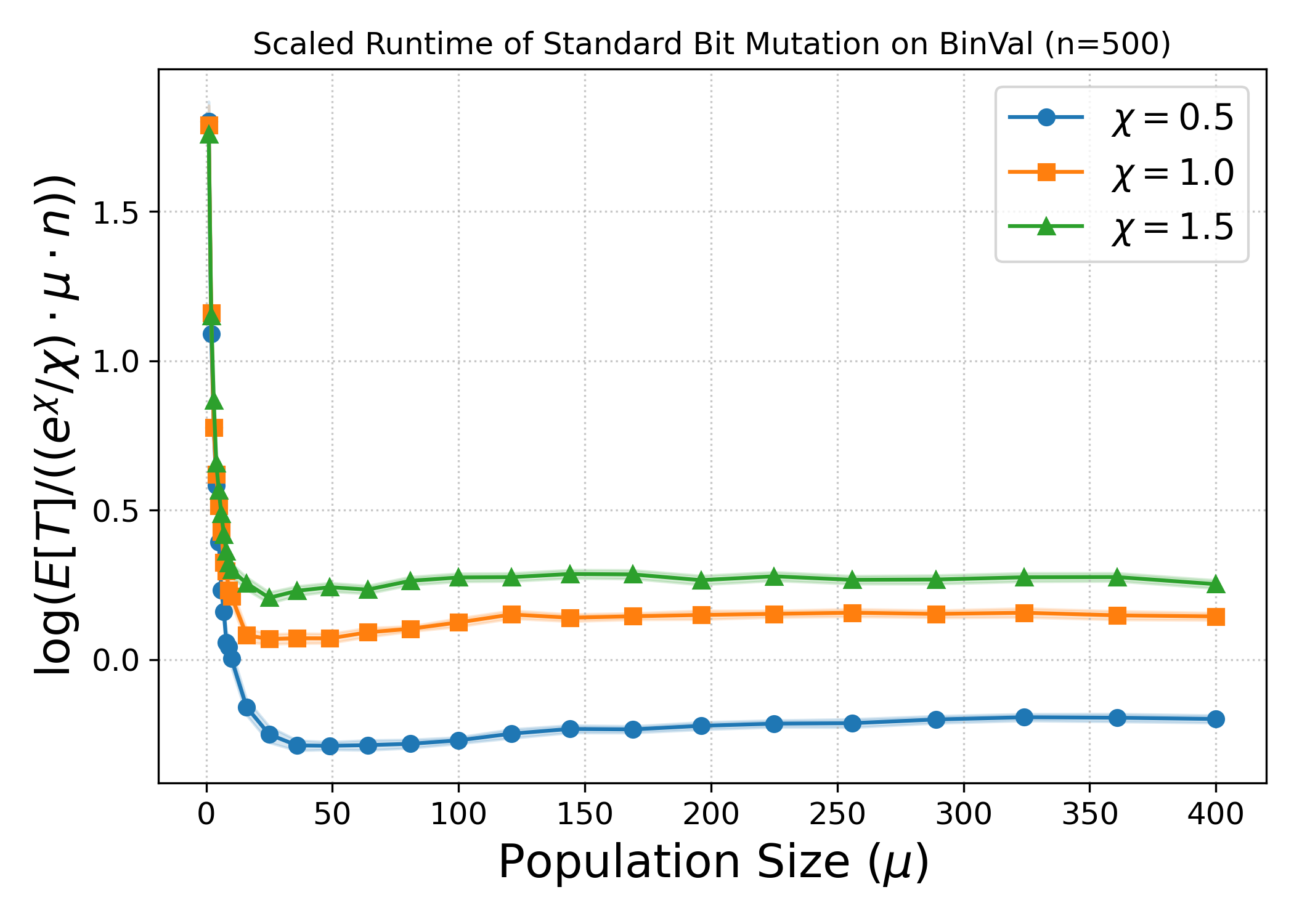}
\caption{}
\end{subfigure}  
\begin{subfigure}{0.49\textwidth}
\includegraphics[width=0.95\textwidth]{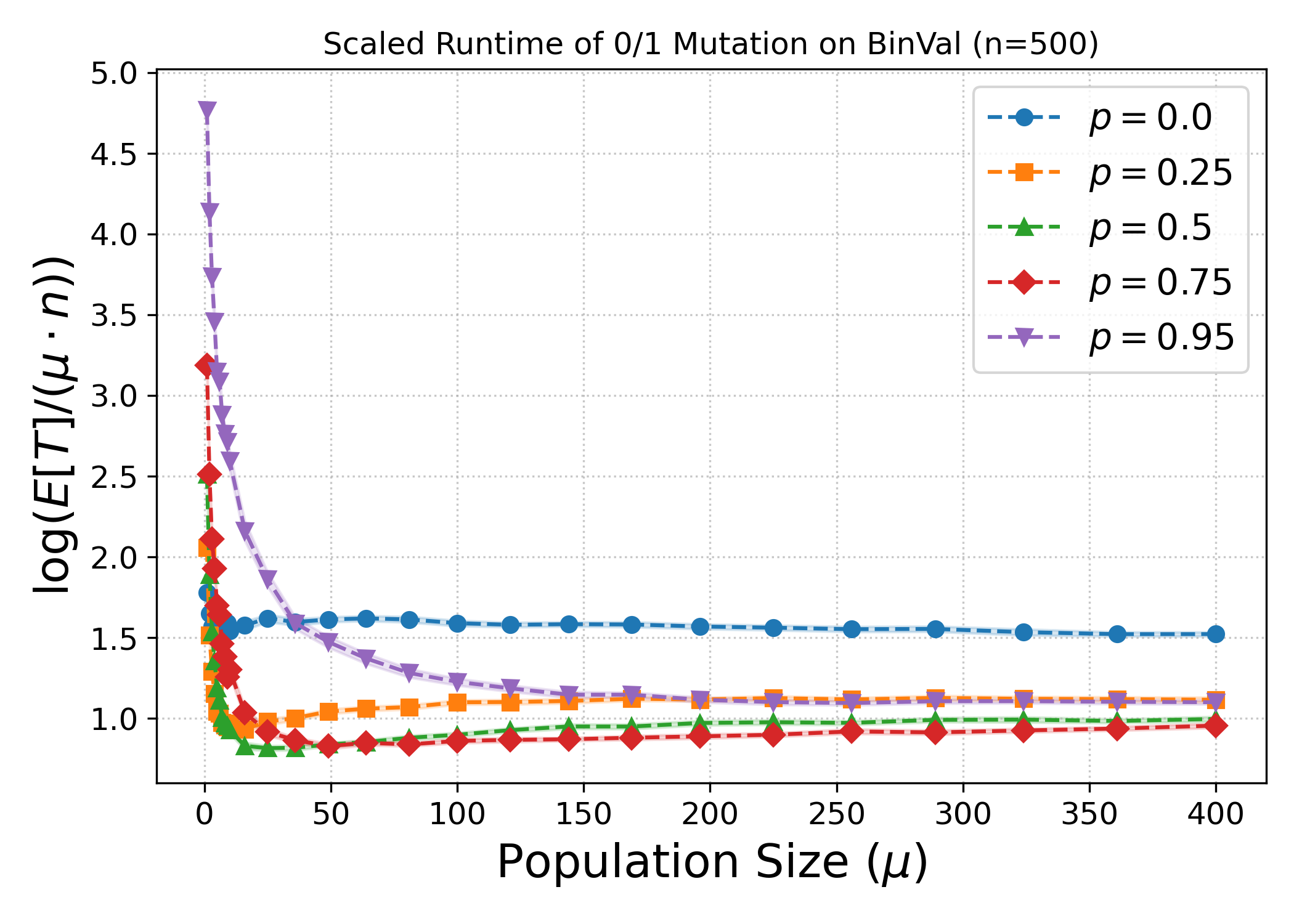}
\caption{}
\end{subfigure}  
    \caption{Logarithm of the expected runtime of the \mpoEA on \binval ($n=500$) plotted against the population size. Also plotted is the 95\% confidence interval as a shaded region. \emph{(a):} Standard bit mutation across varying mutation parameters $\chi$. The expected runtime is scaled by $\tfrac{e^{\chi}}{\chi} \mu n$ and then enclosed in a logarithm. \emph{(b):} 0/1 Mutation for varying zero-flip probabilities $p$. The expected runtime is scaled by $ \mu n$ and then enclosed in a logarithm.}
    \label{fig:empirical_runtimes}
\end{figure}

We observe that the expected runtime scaled by $\mu n$ starts at a large value, especially for $\mu=1$. This case is well understood, the runtime is $ (1+o(1))\tfrac{e^{\chi}}{\chi} n\ln n$. Indeed, an \emph{additive} term $(1+o(1))\tfrac{e^{\chi}}{\chi} n\ln n$ seems inevitable. Consider the smallest Hamming distance from the optimum in the population. Reducing this Hamming distance from $i$ to $i-1$ plausibly takes expected time at least $\approx \tfrac{e^{\chi}}{\chi}\cdot \frac{n}{i}$ for small values of $i$ (the time to reduce it by a single-bit flip), even in the optimistic case that \emph{all} individuals have Hamming distance $i$ from the optimum. Summing this over all $i$ yields the term $(1+o(1))\tfrac{e^{\chi}}{\chi} n\ln n$. However, the fact that the rescaled runtime in Figure~\ref{fig:empirical_runtimes} (a) decays rapidly for $\mu >1$ suggests that this is only an additive term which is not multiplied by $\mu$ in the real runtime.

 As $\mu$ increases further beyond $\mu \approx 10$, the plot stabilizes to a small constant. This seems to suggest that the the factor $\mu\log \mu$ in our theoretical bound is tight up to the $\log \mu$ term, which may be superfluous. In general, the experiments suggest a runtime of $\Theta(\mu n + n \ln n$). We leave it to future work to investigate those hypotheses further.
 

\subsubsection{$0/1$ Mutation.}

Figure~\ref{fig:empirical_runtimes} (b) shows our empirical evaluation for the 0/1 mutation operator. Under this regime, exactly one bit is flipped uniformly at random with probability $1-p$, and zero bits are flipped with probability $p$. This operator explicitly parametrises the generation of exact clones, allowing us to isolate the effect of duplicate creation on the algorithmic slowdown. We simulate the process for varying no-flip probabilities $p \in \{0.0, 0.25, 0.5, 0.75,0.95\}$. 

We observe a similar scaling of the runtime as in the standard-bit mutation, which is line with the similar theoretical results that we got for this setting. For $\mu=1$ we expect a runtime of $(1+o(1))n\ln n/(1-p)$, since the process is a variant of Random Local Search (RLS) that is slowed down by a factor of $(1-p)$. We observe that the runtime is minimized for an intermediate value of $p \approx 0.75$, while both smaller and larger values increase the runtime. Note that increasing $p$ all the way to $1$ would lead to infinite runtimes, since the algorithm could then only produce clones. As the value of $\mu$ increases, we observe that the plots seem to converge to a constant value. Similar to the standard bit mutation, the plot suggests a runtime of $\Theta(\mu n + n\log n)$. 

%% file: chapters/06-conclusion.tex
\section{Conclusion}
\label{sec:conclusion}
We have improved the best known upper bound on the expected runtime of the $(\mu+1)$~EA on \textsc{BinVal} from $O(\mu^5 n \log(n/\mu^4))$~\cite{KrejcaNeumannWitt2025} to $O(\mu \log \mu \cdot n \log n)$ for $\mu = \oh{n / \log n}$, a substantial reduction in the polynomial dependence on the population size~$\mu$. Our analysis builds on the block-based partitioning framework of~\cite{KrejcaNeumannWitt2025} but refines the treatment of interference during the spreading of mutations through the population, by tracking the number of remaining unfixed bits in the current block rather than bounding the interference probability uniformly by the block size. This more precise analysis removes the high-degree polynomial factors that appeared in the earlier bound.

Our improved bound recovers the $O(n \log n)$ runtime of the $(1+1)$~EA on linear functions~\cite{Witt2013} for constant~$\mu$ and grows only mildly with the population size. Since a lower bound of $O(\mu n + n \log n)$ is known for this problem~\cite{Witt2006}, our upper bound is consequently tight up to some $\log$ factors. Whether the factor~$\mu \log \mu$ is tight or can be further reduced --- ideally closer to the additive $\mu n$ overhead observed for \textsc{OneMax}~\cite{Witt2006,AntipovDoerr2021} --- remains an interesting open question. It would also be interesting to re-investigate to related questions that were studied in~\cite{KrejcaNeumannWitt2025}: the runtime of the RLS algorithm, i.e. $0/1$ mutation with $p=0$; and for all algorithms the time to approximate the optimum so that only the least significant $n/\mu$ bits are non-optimal.

On the methodological side, it would be natural to investigate whether the techniques developed here extend to all linear functions, or whether the specific exponential weight structure of \textsc{BinVal} is essential for the analysis. Finally, it would be interesting to see whether the class of conservative mutation operators can be extended to also include crossover operators, so that the runtime bound generalizes to the $(\mu+1)$~GA as well. Moreover, it would be interesting to study whether crossover can provably accelerate optimization of \textsc{BinVal}, as crossover has produced provable speedups on other benchmark families such as Jump~\cite{DangFriedrichKoetzingKrejcaLehreOlivetoSudholtSutton2018,DoerrEcharghaouiJamalKrejca2024}, and related positive evidence exists in dynamic \textsc{BinVal} settings~\cite{LenglerMeier2020}.